\documentclass{article}

\usepackage{microtype}
\usepackage{graphicx}
\usepackage{subcaption}
\usepackage{booktabs} 

\usepackage{hyperref}

\usepackage[accepted]{icml2026}

\usepackage{amsmath}
\usepackage{amssymb}
\usepackage{mathtools}
\usepackage{amsthm}

\usepackage{multirow}
\usepackage{listings}
\usepackage{xcolor}
\usepackage{tikz}
\usepackage{tikz-cd}
\usetikzlibrary{arrows.meta,positioning,shapes.geometric,fit,backgrounds,calc}
\usepackage{enumitem}
\usepackage{float}
\usepackage{xurl}
\usepackage{booktabs}

\usepackage[capitalize,noabbrev]{cleveref}

\theoremstyle{plain}
\newtheorem{theorem}{Theorem}[section]

\newtheorem{corollary}[theorem]{Corollary}
\theoremstyle{definition}
\newtheorem{definition}[theorem]{Definition}

\theoremstyle{remark}

\usepackage[disable,textsize=tiny]{todonotes}

\newcommand{\MI}{\mathrm{MI}}

\lstdefinestyle{dafny}{
    basicstyle=\footnotesize\ttfamily,
    keywordstyle=\bfseries\color{blue!70!black},
    commentstyle=\itshape\color{gray},
    frame=single,
    framerule=0.4pt,
    xleftmargin=0.5em,
    breaklines=true,
    breakatwhitespace=true,
    columns=fullflexible,
    keepspaces=true,
    morekeywords={method,returns,modifies,invariant,while,var,ensures,predicate,ghost,function,datatype,lemma,requires,exists,forall,in,match,case,if,then,else,assert,assume,refines,module,import,opened,include,seq,set,map,nat,int,bool,string,true,false,null},
        literate={<==>}{{\textcolor{blue!70!black}{\textbf{<==>}}}}4
            {==>}{{\textcolor{blue!70!black}{\textbf{==>}}}}3
            {&&}{{\textcolor{blue!70!black}{\textbf{\&\&}}}}2
            {||}{{\textcolor{blue!70!black}{\textbf{||}}}}2,
    }

\icmltitlerunning{Containment Verification}

\begin{document}

\twocolumn[
  \icmltitle{Containment Verification: AI Safety Guarantees \\ Independent of Alignment}



  \icmlsetsymbol{equal}{*}

  \begin{icmlauthorlist}
    \icmlauthor{Royce Moon}{enclave}
    \icmlauthor{Lav R.\ Varshney}{stonybrook}
  \end{icmlauthorlist}

  \icmlaffiliation{enclave}{Enclave Intelligence}
  \icmlaffiliation{stonybrook}{AI Innovation Institute, Stony Brook University}

  \icmlcorrespondingauthor{Royce Moon}{moonr@umich.edu}

  \icmlkeywords{AI safety, AI security, agentic safety, agentic security, agentic systems, long-horizon safety, formal verification, Dafny, large language models}

  \vskip 0.3in
]



\printAffiliationsAndNotice{}  

\begin{abstract}
    Agentic frameworks are the software layer through which AI agents act in the world. Existing safety methods intervene on the model and therefore remain conditional on unverifiable properties of learned behavior. We introduce \textbf{containment verification}, which locates safety guarantees in the agentic framework itself. Under \emph{havoc oracle semantics}, the AI is modeled as an unconstrained oracle over the framework's typed action space, and the verified containment layer must enforce the boundary policy for every typed action value the AI can emit. For \emph{boundary-enforceable} properties, expressed over modeled boundary events, action arguments, and state, we prove a universal guarantee by forward-simulation refinement and mechanize it in Dafny. We instantiate the paradigm by verifying PocketFlow, a minimalist agentic LLM framework, and use an agentic synthesis pipeline to generate the specification, operational model, and refinement proof under an information barrier against tautological specifications. To our knowledge, this is the first deductive formal verification of an agentic framework. The guarantee is independent of alignment because it quantifies over the framework's typed action boundary rather than over model behavior.
\end{abstract}

\section{Introduction}
\label{sec:intro}

AI agents are being deployed to take increasingly consequential actions on real systems, directly affecting the external world. Existing safety approaches for these agents are uniformly conditional on properties of the AI itself, including alignment training~\cite{casper2023rlhf,bai2022constitutional}, mechanistic interpretability~\cite{bricken2023monosemanticity}, faithful chain-of-thought, and adversarial defense protocols~\cite{greenblatt2024control}. However, none of these properties is formally verifiable and empirical evidence shows their underlying assumptions breaking in practice. In recent work, sleeper agents were shown to survive standard safety training~\cite{hubinger2024sleeper}, chain-of-thought was shown to be unfaithful~\cite{lanham2023cot,turpin2023cot}, and adversarial defense protocols reduced attack success but did not completely eliminate it~\cite{ctrlz2025}. As model capability grows, the gap between what an agent can do and what we can say about its safety widens. 

Engineering safety in high-consequence domains has historically relied on fail safes whose guarantees do not depend on the controlled system functioning correctly. In nuclear engineering, a nuclear reactor's control rods drop via gravity when active control is lost. The same construction is realized in process engineering by pressure relief valves that discharge at threshold regardless of root cause, and in computer architecture by hardware memory protection that enforces isolation between processes whatever the running process attempts. We propose \emph{containment verification} as a fail safe for AI agents: rather than attempting to verify properties of the AI itself, we deductively verify the agentic frameworks that mediate between the AI and external state. Production deployments build on agentic frameworks such as LangChain~\cite{langchain}, the OpenAI Agents SDK~\cite{openaiagentssdk}, and the more recent OpenClaw~\cite{openclaw} and Hermes Agent~\cite{hermesagent}. Our instantiation of \emph{containment verification} targets PocketFlow~\cite{pocketflow}, a minimalist LLM framework that exposes the underlying state machine common to agentic frameworks as an explicit graph abstraction over which arbitrary agentic workflows can be composed. As agentic frameworks are conventional software, their behavior can be specified, formally verified, and shown to enforce boundary-event safety for every typed action the AI can emit, independent of the AI's capability, training procedure, or strategy. Henceforth, we refer to agentic frameworks as \emph{containment layers} in containment verification.

The verification regime of containment verification is \emph{havoc oracle semantics}, a standard primitive in deductive verification languages~\cite{barnett2005boogie,leino2008boogie2}. Under havoc semantics, the AI is modeled as an oracle, a procedure with no body whose return value the verifier admits as ranging over the entire action type. The verification obligation reduces to proving that the containment layer enforces the safety properties for every value the AI could possibly emit. The proof is structured as a forward-simulation refinement between an abstract state machine specification of boundary safety and the concrete operational state machine that implements it. 

Constructing the abstract specification, the operational model, and the refinement relation for a containment layer has historically required manual proof engineering at a scale that has constrained formal verification's reach in practice, the engineering bottleneck Dalrymple et al.~\cite{dalrymple2024gsai} identify as a missing piece for guaranteed safe AI.  We address it with an agentic pipeline that synthesizes these formal artifacts with large language models in Dafny, under an architectural information barrier that rules out tautological specifications by construction. The pipeline is itself built on PocketFlow and takes the containment layer's source as input to emit a machine-checked Dafny artifact.

The class of safety properties for which this works is \emph{boundary-enforceable}, meaning predicates over the typed action, modeled boundary event, and system state, with no reliance on the natural-language semantics of the AI's reasoning. The motivating high-consequence failures of agentic systems are boundary failures. Unauthorized network egress, model weight or context exfiltration, irreversible financial actions, destructive filesystem operations, database modification, and tampering with monitoring or training infrastructure all require the agent to cross an effect boundary. Containment verification targets that boundary directly. Under effect exclusivity, every such effect passes through the verified typed action interface, so the AI cannot realize these failures through the containment layer regardless of its capability or strategy. Content-level properties such as truthfulness or non-deception fall outside this scope. The resulting guarantee is universal over all typed action sequences at the modeled boundary and is realized in Dafny~\cite{li2025dafny}.

A separate line of recent work applies formal verification to AI safety, but verifies artifacts adjacent to the AI rather than the containment layer itself. The verification targets range from generated policy code in VeriGuard~\cite{veriguard} to runtime domain-specific language (DSL) monitors in AgentSpec~\cite{agentspec} and information flow controlled interpreters in CaMeL~\cite{camel}. Containment verification, in contrast, verifies the containment layer directly under havoc oracle semantics. Interface discipline efforts such as the Model Context Protocol~\cite{mcp} structure the action surface but provide no formal guarantee. As model capabilities scale, the gap between interface discipline and verified safety widens.

A second line of recent work scales beyond fixed-test benchmarking by computing rigorous probabilistic bounds on how often an LLM violates safety properties under specified input distributions. BEAVER gives deterministic sound bounds on outputs satisfying prefix-closed semantic constraints such as output-format correctness, absence of leaked personally identifiable information (PII), and vulnerability-free code generation~\cite{beaver2025}. C\textsuperscript{3}LLM gives Clopper-Pearson confidence intervals on multi-turn conversational risk, surfacing failure rates that single-turn benchmarks miss~\cite{c3llm2025}. Lumos provides a probabilistic-programming DSL with formal semantics that orchestrates statistical certifiers over prompt distributions, including vision-language model safety in autonomous driving scenarios~\cite{lumos2026}. These methods target the LLM's output distribution whereas containment verification targets the containment layer that mediates between the AI and external state.

\subsection*{Contributions}
\begin{enumerate}
    \item \textbf{The \emph{containment verification} paradigm.} We formalize a fail safe paradigm for AI agents in which boundary-enforceable safety properties are guaranteed by forward-simulation refinement between an abstract specification of boundary safety and the concrete containment layer that implements it, with the AI admitted as an unconstrained oracle under havoc semantics.
    \item \textbf{A soundness theorem.} We prove that, under boundary-event refinement and effect exclusivity, every modeled external effect emitted by the containment layer satisfies the verified policy for every typed action sequence at the modeled boundary (Theorem~\ref{thm:cvs}).
    \item \textbf{A formally verified, deployed PocketFlow instantiation.} We instantiate the theorem with a boundary-event policy for PocketFlow and prove the induced state invariants over the recorded trace in Dafny. To our knowledge, this is the first deductive verification of an agentic framework.
    \item \textbf{An agentic formal synthesis pipeline.} We automate formal specification construction, refinement, and verification via a seven-phase agentic pipeline that is itself built on PocketFlow. Safety property derivation operates under an architectural information barrier that rules out tautological specifications by construction. Resolution, vacuity, and discrimination gates further reject vacuous or non-discriminating specifications, and iterative proof repair drives Dafny to a fixed point.
\end{enumerate}

\section{Containment Verification}
\label{sec:paradigm}
Safety guarantees admit a strength ordering by the set of adversarial scenarios each covers, summarized in Table~\ref{tab:approaches}. The two weaker levels, empirical and distributional, certify respectively against finitely-tested inputs and against stipulated input distributions, with the distributional level approached by alignment training methods such as RLHF~\cite{casper2023rlhf} or Constitutional AI~\cite{bai2022constitutional} when their assumptions hold. Universal guarantees, the strongest level, certify no failures over every value of the action type and are reached by deductive verification under havoc oracle semantics. Letting $\mathcal{S}(\gamma)$ denote the set of adversarial scenarios under which guarantee $\gamma$ makes a claim,
\[
    \mathcal{S}(\text{Empirical}) \subseteq
    \mathcal{S}(\text{Distributional}) \subseteq
    \mathcal{S}(\text{Universal}),
\]
with both inclusions typically strict. Approaches further divide by \emph{locus}, the component on which a safety mechanism acts, into oracle-level methods that act on the AI itself and container-level methods that act in the containment layer that mediates between the AI and external state. No oracle-level approach reaches universal guarantees as each is conditional on a property of the AI's learned behavior.

\begin{table}[H]
    \begin{center}
        \begin{small}
            \begin{tabular}{@{}lll@{}}
                \toprule
                \textbf{Locus} & \textbf{Method} & \textbf{Guarantee} \\
                \midrule
                Oracle    & Red teaming, adversarial evaluation    & Empirical \\
                Oracle    & CoT monitoring, interpretability       & Empirical \\
                Oracle    & AI control protocols                   & Empirical \\
                Oracle    & RLHF, Constitutional AI                & Distributional \\
                Oracle    & Activation steering, circuit breakers  & Distributional \\
                Container & Sandboxing (process, capability)       & Empirical \\
                Container & \textbf{containment verification}      & \textbf{Universal} \\
                \bottomrule
            \end{tabular}
        \end{small}
    \end{center}
    \caption{AI safety approaches by locus and maximum achievable guarantee strength. Oracle-level approaches constrain the AI's learned behavior and their upper bound is distributional. Container-level approaches act in the containment layer that mediates between the AI and external state. Only containment verification has its conditional formally verified.}
    \label{tab:approaches}
\end{table}

The empirical refutations cited in Section~\ref{sec:intro} substantiate the oracle-level rows of the table. Sleeper agents survive standard safety training by distinguishing training from deployment contexts, an ability that does not weaken with capability. Chain-of-thought monitoring presupposes a faithfulness condition between externalized reasoning and internal computation that no current method enforces well. AI control protocols, designed under an explicitly adversarial threat model, bound attack success rates without eliminating them. In other words, no oracle-level approach produces a guarantee of the form ``the system cannot $X$, regardless of how the model tries or its capabilities.''

The two container-level rows of Table~\ref{tab:approaches} act at the same locus but on different layers of the agent deployment stack (Figure~\ref{fig:stack}). Production sandboxes such as Bubblewrap, Firecracker, Docker, or Podman confine the agent process via Linux namespaces, seccomp filters, hypervisor isolation, and capability tokens. Sandboxes operate on the host kernel or hypervisor level and are operationally audited rather than formally verified. The recently disclosed Copy Fail vulnerability, a Linux kernel privilege escalation latent for nearly a decade, empirically illustrates the consequence~\cite{copyfail2026}.

\begin{figure}[t]
\centering
\begin{tikzpicture}[
    every node/.style={font=\footnotesize},
    layer/.style={
        draw, rounded corners=2pt,
        text width=0.85\columnwidth,
        align=center, inner sep=4pt
    },
    ai/.style={layer, fill=orange!12, draw=orange!60!black},
    verified/.style={layer, fill=blue!10, draw=blue!55!black,
        line width=0.9pt},
    empirical/.style={layer, fill=gray!8, draw=gray!55!black},
    flow/.style={-{Stealth[length=4pt, width=4pt]}, thick, gray!60!black},
    node distance=0.55em
]
\node[ai] (ai) {AI model \\
    \textcolor{gray!50!black}{\scriptsize\itshape Modeled as havoc oracle}};
\node[verified, below=1.6em of ai] (cl) {Containment layer \\
    \textcolor{blue!55!black}{\scriptsize\itshape Agentic frameworks (formally verified in containment verification)}};
\node[empirical, below=of cl] (sb) {Sandbox \\
    \textcolor{gray!50!black}{\scriptsize\itshape Bubblewrap, Firecracker, Docker, Podman}};
\node[empirical, below=of sb] (k) {Host kernel or hypervisor \\
    \textcolor{gray!50!black}{\scriptsize\itshape Linux, KVM}};
\node[empirical, below=of k] (hw) {Hardware \\
    \textcolor{gray!50!black}{\scriptsize\itshape Spectre, Meltdown, and other hardware-class vulnerabilities}};
\draw[flow] (ai) -- node[right, midway, font=\scriptsize\itshape,
    text=gray!55!black] {~typed \texttt{Action}} (cl);
\draw[flow] (cl) -- (sb);
\draw[flow] (sb) -- (k);
\draw[flow] (k) -- (hw);
\end{tikzpicture}
\caption{Deployment stack of an AI agent. \texttt{Action} is the typed interface through which the AI's outputs reach the containment layer. The containment layer is formally verified whereas the sandbox, host kernel, and hardware beneath it are operationally audited.}
\label{fig:stack}
\end{figure}
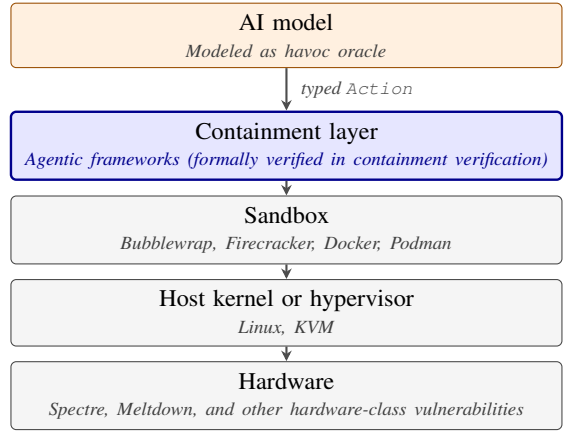

For threats expressed through the verified typed action interface, containment verification is stronger than operational sandboxing as it gives a universal, machine-checked guarantee that no out-of-policy modeled boundary event is emitted for any AI action sequence. Sandboxing confines the process beneath the agent framework and remains complementary, covering lower-level runtime, kernel, and hardware behavior outside the formal model. In a high assurance deployment, the two should be layered where containment verification prevents forbidden effects at the agent boundary, while sandboxing limits damage from bugs or bypasses below that boundary.

Containment verification assumes \emph{effect exclusivity} where every externally visible effect available to the agent is mediated by the containment layer's typed action interface and represented by a modeled boundary event. The theorem is universal over this boundary. Effects outside the interface, including bugs in the Python runtime, operating system, sandbox, or hardware, are outside the formal guarantee and must be handled by complementary isolation mechanisms.

We model the AI in Dafny as a method with an empty body (Listing~\ref{lst:havoc}). The verifier admits any value of the declared return type \texttt{Action} as the model's output. \texttt{Action} is the typed enum enumerating the external effects the containment layer permits the agent to perform which is a property of the containment layer's public interface rather than of the AI. Verification quantifies over every value of \texttt{Action}, a set fixed by the type signature alone and independent of the model's training, capability, or alignment.

\begin{lstlisting}[style=dafny,
    caption={The havoc oracle in Dafny. The external method supplies
        an arbitrary typed \texttt{Action} to the verified containment layer. The oracle itself is not proved safe. The proof shows that, for every possible returned action, the containment layer emits only policy-compliant boundary events.},
    label={lst:havoc}]
method {:extern} GetActionFromLLM(c: Constants, v: Variables)
    returns (action: Action)
    requires v.WF(c)
\end{lstlisting}

The safety object is therefore not the model call but the transition from typed actions to boundary events.

The guarantee is conditional on a refinement structure rather than on a property of the AI's learned behavior or dynamics. A forward-simulation refinement connects an abstract specification of boundary safety to the concrete operational state machine that implements it. An executable Dafny implementation then refines the operational state machine via per-method contracts that Dafny discharges. Both refinements are machine-checked, and the Dafny compiler produces Python from verified source~\cite{li2025dafny}. Since every step in this chain is either machine-proven or a property of deterministic toolchain code, capability invariance follows immediately from havoc's quantification over every value of the action type. Section~\ref{sec:theorem} formalizes the forward simulation and proves the soundness theorem. Specification quality is audited empirically through the three-gate validation of Section~\ref{sec:instantiation}.

\section{Soundness Theorem}
\label{sec:theorem}

We now formalize containment verification as a trace-refinement theorem over boundary events. The theorem is independent of the AI's internal dynamics as the AI is represented only by the typed action value it emits at each step. The guarantee is therefore universal over all oracle strategies expressible through the action type, conditional on a refinement proof and an explicit runtime assumption connecting modeled boundary events to deployed external effects.

\subsection{Labeled transition systems}
\label{sec:ts}

A \emph{labeled transition system} is a tuple
\[
\mathcal{T}=(\mathcal{S},\mathcal{A},\mathcal{E},s^0,\to)
\]
where $\mathcal{S}$ is a state space, $\mathcal{A}$ is the typed action space emitted by the AI, $\mathcal{E}$ is a space of boundary events, $s^0$ is the initial state, and
\[
s \xrightarrow[a]{e} s'
\]
denotes that action $a$ in state $s$ may produce boundary event $e$ and next state $s'$. The transition relation may be nondeterministic. We assume \emph{totality} which entails for every state $s$ and action $a$, there exists at least one pair $(e,s')$ such that $s\xrightarrow[a]{e}s'$. Deterministic step functions are the special case where this pair is unique.

A trace of $\mathcal{T}$ is a sequence $(s_i,a_i,e_i)_{i\ge0}$ such that $s_0=s^0$ and $s_i\xrightarrow[a_i]{e_i}s_{i+1}$ for every $i$. Under \emph{havoc oracle semantics}, the verifier ranges over arbitrary action sequences $a_i\in\mathcal{A}$. The havoc trace set is
\[
\mathrm{Tr}^H(\mathcal{T}) =
\{(s_i,a_i,e_i)_{i\ge0} :
s_0=s^0,\ a_i\in\mathcal{A},\
s_i\xrightarrow[a_i]{e_i}s_{i+1}\ \forall i\}.
\]

The containment layer induces a concrete transition system $\mathcal{T}_I=(\mathcal{S}_I,\mathcal{A},\mathcal{E}_I,s^0_I,\to_I)$. A concrete oracle is any function from finite nonempty concrete histories to actions in $\mathcal{A}$. Stochastic oracles are handled by fixing a realization. When $\to_I$ is nondeterministic, a concrete oracle determines a set of compatible traces rather than a unique trace. Since havoc admits every action sequence, every compatible trace generated by any concrete oracle is covered by the havoc quantification.

\subsection{Boundary-event refinement}
\label{sec:faithfulness}

We use a labeled form of forward-simulation refinement in which each concrete boundary event is matched by a related abstract boundary event. Boundary-enforceable safety is expressed as a transition predicate
\[
\Pi(s,a,e,s')
\]
over the pre-state, AI-issued action, emitted boundary event, and post-state. State invariants are the special case in which $\Pi$ ignores $a$ and $e$. Rejected actions are represented by no-effect events, so the theorem constrains what the containment layer emits and not merely what it records.

\begin{definition}[Boundary-event refinement]
\label{def:boundary-refinement}
Let
\[
\mathcal{T}_M=(\mathcal{S}_M,\mathcal{A},\mathcal{E}_M,s^0_M,\to_M)
\quad\text{and}\quad
\mathcal{T}_I=(\mathcal{S}_I,\mathcal{A},\mathcal{E}_I,s^0_I,\to_I)
\]
share action space $\mathcal{A}$. The implementation $\mathcal{T}_I$ refines the abstract model $\mathcal{T}_M$ with respect to $(\Pi_M,\Pi_I)$ iff there exist relations $\mathcal{R}\subseteq\mathcal{S}_M\times\mathcal{S}_I$ and $\mathcal{Q}\subseteq\mathcal{E}_M\times\mathcal{E}_I$ such that:
\begin{description}[style=nextline,leftmargin=2em]
\item[\normalfont(R1) Initial.]
$(s^0_M,s^0_I)\in\mathcal{R}$.

\item[\normalfont(R2) Step simulation.]
For every $(s_M,s_I)\in\mathcal{R}$ and concrete step $s_I\xrightarrow[a]{e_I}_I s'_I$, there exist $s'_M$ and $e_M$ such that
\[
s_M\xrightarrow[a]{e_M}_M s'_M,\qquad
(s'_M,s'_I)\in\mathcal{R},\qquad
(e_M,e_I)\in\mathcal{Q}.
\]

\item[\normalfont(R3) Safety preservation.]
For every matched pair of steps satisfying (R2),
\[
\Pi_M(s_M,a,e_M,s'_M)\Rightarrow
\Pi_I(s_I,a,e_I,s'_I).
\]
\end{description}
\end{definition}

\begin{figure}[h]
\centering
\begin{tikzcd}[column sep=3.5em, row sep=2.6em]
s_M \arrow[r, "{a/e_M}"]
    \arrow[d, no head, "\mathcal{R}"']
& s'_M \arrow[d, no head, "\mathcal{R}"] \\
s_I \arrow[r, "{a/e_I}"']
& s'_I
\end{tikzcd}
\vspace{-0.5em}
\[
(e_M,e_I)\in\mathcal{Q}
\qquad
\Pi_M(s_M,a,e_M,s'_M)\Rightarrow
\Pi_I(s_I,a,e_I,s'_I)
\]
\caption{Boundary-event refinement. Each concrete step labeled by action $a$ and event $e_I$ is matched by an abstract step with the same action, a related event $e_M$, and related post-state.}
\label{fig:simulation}
\end{figure}

\subsection{Soundness theorem}
\label{sec:thm-statement}

\begin{theorem}[Containment Verification Soundness]
\label{thm:cvs}
Let $\mathcal{T}_M$ and $\mathcal{T}_I$ be total labeled transition systems sharing action space $\mathcal{A}$. Assume:
\begin{description}[style=nextline,leftmargin=2em]
\item[\normalfont(A1) Abstract havoc safety.]
For every havoc trace $(s^M_i,a_i,e^M_i)_{i\ge0}\in
\mathrm{Tr}^H(\mathcal{T}_M)$,
\[
\Pi_M(s^M_i,a_i,e^M_i,s^M_{i+1})
\]
holds for every $i\ge0$.

\item[\normalfont(A2) Boundary-event refinement.]
$\mathcal{T}_I$ refines $\mathcal{T}_M$ with respect to $(\Pi_M,\Pi_I)$ in the sense of Definition~\ref{def:boundary-refinement}.
\end{description}
Then every trace of $\mathcal{T}_I$ compatible with any concrete oracle satisfies
\[
\Pi_I(s^I_i,a_i,e^I_i,s^I_{i+1})
\]
for every step $i\ge0$.
\end{theorem}

\begin{proof}
Fix an arbitrary concrete oracle and an arbitrary compatible concrete trace $(s^I_i,a_i,e^I_i)_{i\ge0}$. By (R1), $(s^M_0,s^I_0)\in\mathcal{R}$. Inductively apply (R2) to each concrete step $s^I_i\xrightarrow[a_i]{e^I_i}_I s^I_{i+1}$, obtaining an abstract step $s^M_i\xrightarrow[a_i]{e^M_i}_M s^M_{i+1}$ with $(s^M_{i+1},s^I_{i+1})\in\mathcal{R}$ and $(e^M_i,e^I_i)\in\mathcal{Q}$. This constructs an abstract trace using exactly the concrete oracle's action sequence. Since each $a_i\in\mathcal{A}$, the constructed trace lies in $\mathrm{Tr}^H(\mathcal{T}_M)$. By (A1), $\Pi_M(s^M_i,a_i,e^M_i,s^M_{i+1})$ holds for every $i$. By (R3), $\Pi_I(s^I_i,a_i,e^I_i,s^I_{i+1})$ holds for every $i$.
\end{proof}

\begin{corollary}[Deployed Boundary Safety]
\label{cor:deployed}
If every externally visible boundary effect of the deployed runtime is represented by some concrete event $e_I\in\mathcal{E}_I$, then every externally visible boundary effect produced under any concrete oracle satisfies the concrete boundary policy $\Pi_I$.
\end{corollary}

\subsection{Trusted base}
\label{sec:trust}

Theorem~\ref{thm:cvs} separates machine-checked proof obligations from the trusted base. Abstract havoc safety (A1) is discharged by Dafny through the initial-state and inductive-step lemmas for the abstract transition system. Boundary-event refinement (A2) is discharged by the refinement proof connecting the concrete operational model to the abstract model. The corollary's effect-faithfulness assumption is the runtime boundary condition that every deployed external effect passes through the modeled event interface.

After verification, the trusted base consists of the concrete boundary policy $\Pi_I$, the abstraction and event relations defining $\mathcal{R}$ and $\mathcal{Q}$, the effect-faithfulness boundary connecting concrete events to deployed external effects, the Dafny verifier and compiler, and the small extern shim that supplies arbitrary typed actions at runtime. The AI model, prompting procedure, training process, and action-selection strategy are not trusted.

\subsection{Mechanization}
\label{sec:artifact}

The proof of Theorem~\ref{thm:cvs} is mechanized in Dafny following the IronFleet idiom~\cite{hawblitzel2015ironfleet}. The artifact uses relational transition predicates rather than deterministic step functions where \texttt{Next(c, v, v', evt)} relates a pre-state, post-state, and boundary event. Dispatch and stutter cases are both represented as events, with stutter modeling rejected actions and early returns.

The PocketFlow instantiation discharges Definition~\ref{def:boundary-refinement} through two refinement lemmas. \texttt{RefinementInit} establishes initial-state matching. \texttt{RefinementNext} shows that every concrete \texttt{Next} step maps, via state abstraction and event abstraction, to an abstract \texttt{BoundarySpec.Next} step while preserving the inductive invariant. The relevant signatures are listed in Appendix~\ref{app:artifact}.

The mechanized end-to-end lemma \path{ContainmentVerificationSoundness} composes three ingredients: 
\begin{itemize}
    \item \path{TraceRefinesSpec} which lifts each concrete trace to an abstract havoc trace by reusing the concrete action sequence.
    \item \path{SpecSafetyAlongTrace} which establishes the abstract boundary predicate pointwise along the lifted trace.
    \item \path{SpecSafetyLiftsToImpl} which applies the state and event abstractions to obtain the concrete boundary predicate at each implementation step.
\end{itemize}

A deployed safety failure can therefore arise from three sources. First, failure can arise from a failed refinement proof which Dafny detects by refusing to discharge the obligations. Second, failure can arise from an effect-faithfulness gap where some deployed external effect bypasses the modeled event interface. Finally, failure can arise from a property-completeness gap where $\Pi_I$ does not capture the intended safety policy.

\section{Instantiation: Formally Verifying PocketFlow}
\label{sec:instantiation}

Theorem~\ref{thm:cvs} takes a formal model $\mathcal{T}_M$, a concrete model $\mathcal{T}_I$, boundary policies $(\Pi_M,\Pi_I)$, and state and event abstraction relations $(\mathcal{R},\mathcal{Q})$ as inputs and yields a universal guarantee over the action type. This section describes how those inputs are produced and audited for the instantiating containment verification to formally verify PocketFlow, and how the verified Dafny source is compiled into the deployed Python runtime.

\subsection{The witness containment layer}
\label{sec:witness}

PocketFlow~\cite{pocketflow} is a minimal LLM framework that exposes the dispatch loop common to LangChain, LangGraph, and related frameworks as an explicit graph abstraction over typed \texttt{Node} classes. A \emph{Flow} in PocketFlow is a directed graph of \texttt{Node} instances connected by labeled action edges. The runtime executes a Flow by dispatching from the current Node along the edge whose label matches the action the current Node emits. It is possible for one to create arbitrarily complex agentic workflows using PocketFlow, making it a great illustrative instantiation for containment verification. We verify PocketFlow's dispatch loop with respect to a four-variant action interface:
\begin{lstlisting}[basicstyle=\small\ttfamily,
                   xleftmargin=1em,  
                   aboveskip=4pt,
                   belowskip=4pt,
                   numbers=none]
datatype Action =
    | NoAction
    | ReadPathAction(path: string)
    | ToolCallAction(tool: string)
    | StepAction
\end{lstlisting}
The PocketFlow instantiation enforces a boundary-event policy with three induced state invariants. Read events are permitted only for paths under the configured workspace root, tool-call events only for allowlisted tools, and step events only while the loop bound has not been exceeded. The recorded state invariants in the artifact are corollaries of this event policy where every recorded read path is workspace rooted, every recorded tool call is allowlisted, and the step counter is bounded. The deployed runtime is exercised on action sequences that include both permitted and non-permitted values, and records only permitted events. Consequently, the deployed Python rejects every action that would emit an out-of-policy boundary event.

\subsection{Agentic formal specification synthesis pipeline}
\label{sec:pipeline}

Constructing $\mathcal{T}_M$, $\mathcal{T}_I$, $(\Pi_M,\Pi_I)$, and $(\mathcal{R},\mathcal{Q})$ for an unfamiliar containment layer is the engineering bottleneck identified by Dalrymple et al.~\cite{dalrymple2024gsai}. We address it with a seven-phase agentic pipeline that takes a Flow file as input and emits a machine-checked Dafny artifact. The pipeline is itself a PocketFlow graph. Table~\ref{tab:phases} summarizes the phases.

\begin{table}[H]
\centering
\begin{small}
\begin{tabular}{@{}p{0.28\columnwidth}p{0.55\columnwidth}l@{}}
\toprule
\textbf{Phase} & \textbf{Output} & \textbf{Tier} \\
\midrule
ContainmentSpec   & Composition graph and per-edge dispatch obligations & S+R \\
InterfaceContract & Typed Node interface contract $\tau$ & S+R \\
GenerateSpec      & Per-path Dafny bundle including $\Pi_M$, $\mathcal{R}$, $\mathcal{Q}$, and $\mathrm{Inv}$ & S+R \\
ValidateSpec      & Three-gate validation outcome & R \\
Verify            & Dafny discharge of proof obligations & --- \\
ClassifyFailure   & Failure category for proof-repair routing & R \\
ProofRepair       & Refinement of the proof file only & R \\
\bottomrule
\end{tabular}
\end{small}
\caption{Pipeline phases. Tier ``S'' uses Claude Sonnet 4.6 for codebase exploration and ``R'' uses Claude Opus 4.7 for synthesis and reasoning. For each dispatch path, the pipeline runs GenerateSpec, ValidateSpec, and Verify in sequence. ProofRepair feeds back into ValidateSpec up to a configured iteration bound.}
\label{tab:phases}
\end{table}

The closed-loop structure of formal specification synthesis, validation, and verifier-feedback repair has antecedents in current LLM-synthesized formal specification work: VeriAct~\cite{veriact2026} uses a similar verify-feedback agentic loop for synthesizing formal specifications of Java methods, and ATLAS~\cite{atlas2026} decouples contract generation from implementation synthesis to prevent post-hoc specification weakening. Our pipeline differs in target (an agentic framework's dispatch loop rather than general code), in the architectural information barrier on property derivation, and in per-path verification aligned with the Flow's compositional graph.

\subsection{Architectural information barrier}
\label{sec:barrier}

Within GenerateSpec, the property-derivation step that produces $P_M$ runs with no implementation context. Letting $R = (\tau, \pi)$ denote the interface contract and the property template, the resulting predicate $P = f_1(R, \theta)$ is a deterministic function of $R$ and the model weights $\theta$ alone, so the conditional mutual information $\MI(P;\, I \mid R, \theta) = 0$ holds by construction. This rules out the failure mode in which an LLM with implementation access synthesizes a tautological specification that redescribes $I$ rather than constraining it. The subsequent model-construction step receives the locked $P_M$ along with implementation context, but cannot weaken what the property already requires.

A related decoupling strategy appears in ATLAS~\cite{atlas2026}, where contracts are frozen before implementation synthesis to prevent the implementation phase from softening earlier specifications. Our barrier is structurally analogous but addresses a different failure mode: it prevents the property-derivation step from observing the implementation in the first place, ruling out tautological specifications by construction rather than by post-hoc enforcement. Section~\ref{sec:empirical} reports the empirical findings.

For ablation, the pipeline accepts a flag that swaps the property-derivation prompt to a variant which receives the repository snapshot and the prior exploration alongside $\tau$ and $\pi$. The barrier-side construction is otherwise unchanged, so the comparison isolates the input-channel restriction rather than confounding it with prompt-engineering changes. Section~\ref{sec:empirical} reports the empirical findings.

\subsection{Three-gate validation}
\label{sec:gates}

A specification that verifies under Dafny might still be vacuous (true of any implementation) or non-discriminating (true of a deliberately unsafe variant). Drawing on IronSpec's mutation-based specification validation~\cite{ironspec}, we adapt the methodology to the LLM-synthesized-specification setting via three gates that all must pass before the final Dafny verification attempt. 

\begin{description}[leftmargin=2em, topsep=4pt, itemsep=4pt, parsep=0pt]
\item[\normalfont(G1) Resolution.] 
    The synthesized Dafny parses and type-checks within a thirty second timeout.
\item[\normalfont(G2) \textit{Vacuity.}]
    Adapting vacuity detection from temporal model checking~\cite{beer2001vacuity} to Hoare-style refinement obligations, a permissive-stub proof variant in which the inductive invariant is gutted to its well-formedness clause alone and the lemma bodies are emptied must fail verification when concatenated with the locked trusted scaffolding. A specification that admits this mutation cannot demand the structural properties refinement requires.
\item[\normalfont(G3) \textit{Discrimination.}]
    Following the mutation-testing methodology of IronSpec~\cite{ironspec}, an LLM-generated mutation that introduces a plausible modeling error inside the proof file must also fail verification. A specification that accepts the alive mutation cannot distinguish faithful from faulty refinements.
\end{description}

Both mutations target only the proof file. The trusted scaffolding and the main theorem are held byte-identical, which restricts the mutation surface to the abstraction and the inductive strengthening, where specification-quality bugs are most likely to hide.

When a gate fails, the pipeline routes to ProofRepair, which is permitted to refine the proof file only and loops back to ValidateSpec up to a configured iteration bound. When the final Dafny verification fails after the gates pass, ClassifyFailure labels the failure as a confirmed open obligation, a missing proof annotation, or a modeling error, and routes accordingly.

\subsection{From verified Dafny to deployed Python}
\label{sec:deploy}

Dafny compiles to Python through its native backend. The havoc oracle in the source Dafny is declared as an external method, which in the compiled Python becomes a call into a small shim module that defers to a runtime-supplied action provider. The provider accepts an arbitrary callable that emits values of the \texttt{Action} type. The oracle itself is not proved safe. The proof shows that, for every possible returned \texttt{Action}, the verified dispatch loop emits only policy-compliant boundary events.

The deployed guarantee relies on effect exclusivity where the runtime exposes no agent-accessible external effect except through the compiled verified dispatch loop and its typed \texttt{Action} interface. The raw model output parser, the extern shim, and the action provider therefore form the runtime type boundary. They are trusted only to supply elements of \texttt{Action}. Safety is enforced after that boundary by the verified transition from actions to boundary events.

The trusted base of the deployed system consists of the concrete boundary policy $\Pi_I$, the state and event abstractions defining $\mathcal{R}$ and $\mathcal{Q}$, the effect-exclusivity boundary connecting deployed effects to modeled events, the runtime type boundary and extern shim, and the Dafny verifier and compiler. A behavioral test suite exercises the compiled runtime under action streams that mix permitted and non-permitted values. In every case the runtime records only permitted events, giving a discrimination-style audit of the compilation pass at the Python level.

\subsection{Empirical observations}
\label{sec:empirical}

We ran the pipeline against two witness Flows for a total of four runs. The first is a file-reading agent that uses the four-variant action interface of Section~\ref{sec:witness} directly. The second is an agentic retrieval-augmented generation Flow whose witness fields are not represented in the trusted property template. Every dispatch path passed all three gates, and Dafny discharged every final proof obligation. Per-path Dafny verification times ranged from 1.2 to 2.5 seconds. Total token usage per run ranged from 1.6 million to 3.4 million tokens, with the no-barrier mode using 25--56\% more tokens because the property derivation step receives the full repository snapshot and prior exploration in addition to the trusted scaffolding.

The information barrier produces a structurally visible effect on the synthesized specifications, but the effect depends on the alignment between the trusted property template and the witness Flow's semantic surface. When the template's fields directly model the Flow's privileged effects, barrier and no-barrier specifications differ by a small number of implementation-derived elements: an additional history-consistency conjunct, an entry node specialization in the initial state, and verbatim source line citations in comments. None of these render the no-barrier specification tautological, but they are observable traces of the implementation in the synthesized predicate.

When the template does not match the Flow, barrier and no-barrier modes diverge qualitatively. The barrier mode is forced to encode the Flow's semantics under a renaming homomorphism. The synthesized transition writes a template field as a stand in for an analogous Flow-specific field, and the corresponding boundary predicate is non-vacuously discharged in the template's vocabulary. The no-barrier mode, with implementation access, recognizes the schema mismatch and abstains where the template field is left unwritten and the corresponding predicate becomes vacuously true. Both modes pass all three gates as the gates target the protocol-invariant portion of the proof rather than the corollary portion in which the silent abstention hides. This is a property-completeness gap of the kind discussed in Section~\ref{sec:artifact} and bounds the strength of the three gates. The gates are necessary but not sufficient for specification quality, and an additional template fitness check is needed to catch silent abstention.

\section{Discussion}
\label{sec:discussion}

The guarantee that containment verification provides is structurally distinct from alignment-based approaches across the AI safety landscape. Every oracle-based approach has a conditional assumption that grows harder to verify and its refutations sharpen as model capability increases. Containment verification's conditional, by contrast, is correspondence between the formal specification and the deployed runtime, a property of deterministic code that does not weaken with model capability. Consequently, a verified containment layer cannot be defeated through the modeled typed action boundary by a more capable AI, it can only be defeated by divergence between the formal specification and the running implementation and abstraction gaps in the specification itself. This divergence, however, has an audit trail and a remediation path that capability scaling does not affect.

This capability-invariance applies directly to the threat vectors motivating AGI and ASI safety concerns~\cite{carlsmith2024xrisk,ngo2025alignment} where exfiltration of model weights or sensitive context to non-allowlisted destinations, irreversible financial and system actions, modification of monitoring or training infrastructure, and unauthorized network egress are examples of the most critical ones. Each decomposes into a boundary action and is therefore boundary-enforceable in the sense of Section~\ref{sec:intro}. A containment layer with a sufficiently narrow action interface renders out-of-policy effects unrepresentable at the verified effect boundary. Under effect exclusivity, this holds regardless of how completely the AI has reverse-engineered the underlying systems or planned the exploit internally. The verified guarantee is foreclosure at the vocabulary level instead of interception at the runtime level. An agent that cannot issue an action cannot succeed at it.

Two design consequences follow. The strength of the guarantee on a given containment layer is determined by the narrowness of the action interface its author chose to expose, so the design pressure shifts from training-time and deployment-time alignment to framework-time interface design. A containment layer that exposes a narrow typed enum yields a strong guarantee, one that exposes a generic shell or eval primitive collapses the guarantee onto the primitive. Containment verification as a paradigm is complementary to the existing AI safety landscape. Oracle-level approaches address the content-level properties such as truthfulness, non-deception, or persuasion resistance that lie outside containment verification's scope. Sandboxing, in this sense, also remains as a fail safe for the residuals discussed in Section~\ref{sec:trust}. A safety-critical deployment should combine all the aforementioned and more.

\subsection{Limitations}
\label{sec:limitations}

\paragraph{Effect exclusivity.} 
The deployed guarantee requires every agent-accessible external effect to pass through the verified typed action boundary. Direct Python calls, foreign-function interfaces, ambient filesystem or network access, concurrency effects, runtime bugs, kernel vulnerabilities, and hardware faults are outside the theorem unless mediated by that boundary. High assurance deployments should therefore pair containment verification with sandboxing, capability isolation, and review of the small runtime type boundary.

\paragraph{Template fitness.}
The information barrier rules out one failure mode which is tautological specifications that redescribe the implementation. The ablation in Section~\ref{sec:empirical} surfaces a second failure mode that the barrier does not address. When the trusted property template's \texttt{Variables} schema does not include a field corresponding to a witness Flow's privileged effect, the LLM with implementation access leaves the analogous template field unwritten in the synthesized transition rather than encode the effect under a renaming. The corresponding boundary predicate then becomes vacuously true, because the field it quantifies over is empty along every reachable state. The verifier discharges the specification, the three gates pass, and the runtime safety surface is unmonitored. The gates do not catch this because they audit the protocol-invariant portion of the proof rather than whether each conjunct of the safety predicate is non-vacuously discharged. An additional template-fitness check, validating that every safety corollary is reachable-non-empty before the final verification attempt, is needed to close this gap.

\paragraph{Specification fidelity.} 
Theorem~\ref{thm:cvs} establishes safety relative to the operational state machine $\mathcal{T}_I$, not relative to the running Python. Boundary-event refinement requires faithfulness between $\mathcal{T}_M$ and $\mathcal{T}_I$, which the Dafny refinement lemmas discharge. The operational state machine itself does not formalize concurrency, side channels, timing, or hardware behavior, so no formal proof connects it to the deployed runtime at those levels. This parallels seL4's hardware-fault exclusion~\cite{klein2009sel4} and is a standard limitation of all formal verification.

\paragraph{Boundary-enforceable scope.}
Containment verification covers only properties expressible as predicates over typed actions, modeled boundary events, action arguments, and system state. Content-level properties such as truthfulness, non-deception, non-manipulation, and the like require interpretation of the AI's output semantics that the boundary-enforceable framing discards by design and therefore fall outside this scope. Addressing them requires complementary methods that operate on the model's generation distribution rather than on its action interface.

\paragraph{Witness scope.}
The verified guarantee on PocketFlow is at the strength of PocketFlow's four-variant action interface. Frameworks exposing wider interfaces require correspondingly more proof cases. The paradigm scales with the framework author's interface design rather than with the verification mechanism.

\subsection{Future work}
\label{sec:future}

Containment verification opens a path to securing production AI agents at the action-interface level today and at scale as model capabilities grow. The natural next step is extending the synthesis pipeline and trusted scaffolding to other larger production agentic frameworks such as LangChain. The deliberately minimal PocketFlow witness establishes that the verification holds for a four-variant action interface, and the engineering test of the paradigm is whether it scales to interfaces with dozens of variants and richer dispatch logic. Adaptive per-flow templates synthesized from the interface contract, under the same architectural information barrier, would close the silent-abstention gap of Section~\ref{sec:empirical}  and remove the alignment dependency between trusted scaffolding and witness vocabulary that limits the present pipeline. A formal composition lemma relating the universal-boundary guarantee to the distributional-content guarantees of statistical certification~\cite{beaver2025,c3llm2025,lumos2026} would let deployments combine the two formally rather than operationally. As AI capability scales, the conditional assumption of every oracle-level approach grows harder to verify. The conditional assumption of containment verification, the correspondence between the formal specification and the deployed runtime, becomes the residual question. And that question scales with engineering effort on the containment layer rather than with the AI's capability. 

\section*{Impact Statement}

This paper introduces containment verification, a fail-safe paradigm whose safety guarantee for AI agents is invariant to model capability, alignment provenance, and training distribution. The intended societal impact is a deployment-time safety mechanism for the threat vectors central to current and projected AI safety concerns.

The guarantee is structurally distinct from alignment-based approaches, and we want to be precise about what it does and does not cover. Containment verification addresses only boundary-enforceable properties, predicates over typed actions, modeled boundary events, action arguments, and system state. Content-level properties such as truthfulness, non-deception, and non-manipulation lie outside this scope by design and require complementary methods that operate on the model's generation distribution. A verified containment layer is therefore not a complete safety solution. It is a fail safe for a specific class of harms, intended to compose with oracle-level methods, sandboxing, and statistical certification rather than replace any of them. The societal claim is therefore not that containment verification solves all agent safety, but that it gives a capability-invariant guarantee for the subset of harms expressible and exclusively mediated at the typed effect boundary. Safety critical deployments should layer all of these in addition to red teaming all necessary components.

The paradigm shifts safety design pressure from training-time and deployment-time alignment to framework-time interface design. A containment layer that exposes a narrow typed enum yields a strong guarantee. The societal payoff therefore depends on framework authors choosing narrow interfaces, which entails an opinionated and visible stance on what agents should be permitted to do. We view the visibility and contestability of this design choice, relative to the opacity of alignment training, as a feature.

The agentic formal specification synthesis pipeline of Section~\ref{sec:pipeline} is itself an LLM-driven capability that could be applied to verify containment layers for systems whose deployment is contested. It is general-purpose verification infrastructure, and its societal valence inherits from the deployer's intent rather than from the technique itself. The same methodology applies equally to current production deployments where boundary-enforceable safety failures already occur and to the AGI and ASI class threat models that motivate the paradigm's capability-invariance.

\bibliography{containmentveri}
\bibliographystyle{icml2026}

\newpage
\appendix
\onecolumn
\section{Verification Artifact Excerpts}
\label{app:artifact}

The mechanized proof of Theorem~\ref{thm:cvs} is realized across six Dafny modules:
\begin{itemize}
    \item \texttt{Types.t.dfy}: the action and event datatypes.
    \item \texttt{BoundarySpec.t.dfy}: the abstract transition system and boundary safety predicate.
    \item \texttt{Containment.t.dfy}: the operational transition system $\mathcal{T}_I$.
    \item \texttt{RefinementObligation.t.dfy}: the abstract module whose signatures appear in Listing~\ref{lst:refinement}.
    \item \texttt{RefinementProof.v.dfy}: the concrete refinement that discharges Definition~\ref{def:boundary-refinement}.
    \item \texttt{MainTheorem.v.dfy}: the end-to-end soundness lemma.
\end{itemize}

The submitted paper includes the proof-relevant artifact excerpts below. The full source, synthesis pipeline, and witness Flow corpus are maintained as a separate implementation artifact.

\subsection{Refinement obligations}
\label{app:refinement-obligations}
 
The PocketFlow proof discharges the boundary-event refinement obligations of Definition~\ref{def:boundary-refinement} through state and event abstraction. \texttt{ConstantsAbstraction} and \texttt{VariablesAbstraction} induce the state relation $\mathcal{R}$, \texttt{EventAbstraction} induces the event relation $\mathcal{Q}$, and \texttt{Inv} supplies the inductive strengthening needed by the proof.

\begin{lstlisting}[style=dafny,
    caption={Refinement obligation excerpt. The two lemmas discharge
        initial-state matching and boundary-event step simulation.},
    label={lst:refinement}]
ghost function ConstantsAbstraction(c: Constants)
    : BoundarySpec.Constants

ghost function VariablesAbstraction(c: Constants, v: Variables)
    : BoundarySpec.Variables
    requires v.WF(c)

ghost function EventAbstraction(evt: Event)
    : BoundarySpec.Event

ghost predicate Inv(c: Constants, v: Variables)

lemma RefinementInit(c: Constants, v: Variables)
    requires Init(c, v)
    ensures Inv(c, v)
    ensures BoundarySpec.Init(
        ConstantsAbstraction(c),
        VariablesAbstraction(c, v))

lemma RefinementNext(c: Constants, v: Variables,
                     v': Variables, evt: Event)
    requires Next(c, v, v', evt) && Inv(c, v)
    ensures Inv(c, v')
    ensures BoundarySpec.Next(
        ConstantsAbstraction(c),
        VariablesAbstraction(c, v),
        VariablesAbstraction(c, v'),
        EventAbstraction(evt))
\end{lstlisting}

\subsection{The safety predicate $P_M$}
\label{app:safety}

\texttt{Safety} is the boundary-enforceable safety predicate declared in \texttt{BoundarySpec.t.dfy}. It conjoins three predicates on the boundary fields: workspace-rooted read paths, allowlisted tool calls, and a bounded step count.

\begin{lstlisting}[style=dafny,breaklines=true,basicstyle=\footnotesize\ttfamily]
ghost predicate Safety(c: Constants, v: Variables) {
    && (forall i :: 0 <= i < |v.read_paths| ==>
        StartsWith(v.read_paths[i], c.workspace_root))
    && (forall i :: 0 <= i < |v.tool_calls| ==>
        v.tool_calls[i] in c.allowed_tools)
    && v.step_count <= c.max_steps
}
\end{lstlisting}

\subsection{The relation $\mathcal{R}$ and $\mathcal{Q}$}
\label{app:relation}

The state relation $\mathcal{R}$ is induced by \texttt{VariablesAbstraction} together with the inductive strengthening \texttt{Inv}. The event relation $\mathcal{Q}$ is induced by \texttt{EventAbstraction}. \texttt{Inv} captures the structural invariants the refinement requires which are well-formedness, the step-count upper bound, a halted-implies-bound clause, an alignment between \texttt{history} length and \texttt{step\_count}, and a history-tail consistency clause.

\begin{lstlisting}[style=dafny,breaklines=true,basicstyle=\footnotesize\ttfamily]
ghost function ConstantsAbstraction(c: Constants)
    : BoundarySpec.Constants
{
    BoundarySpec.Constants(c.workspace_root, c.allowed_tools, c.max_steps)
}

ghost function VariablesAbstraction(c: Constants, v: Variables)
    : BoundarySpec.Variables
{
    BoundarySpec.Variables(v.read_paths, v.tool_calls, v.step_count, v.halted)
}

ghost predicate Inv(c: Constants, v: Variables)
{
    && v.WF(c)
    && v.step_count <= c.max_steps
    && (v.halted ==> v.step_count >= c.max_steps)
    && |v.history| == v.step_count
    && (v.last_node != NoNode ==>
        |v.history| > 0
        && v.history[|v.history| - 1] == (v.last_node, v.last_action))
}
\end{lstlisting}

\subsection{The soundness lemma}
\label{app:soundness}

\texttt{ContainmentVerificationSoundness} is the trace-level mechanization of Theorem~\ref{thm:cvs} for the PocketFlow witness. The listed \texttt{ensures} clauses are the state-level consequences of the boundary-event safety theorem. The theorem first proves that each modeled boundary event satisfies the concrete boundary policy $\Pi_I$. The recorded-state invariants follow because permitted events are the only events appended to the trace.

\begin{lstlisting}[style=dafny,breaklines=true,basicstyle=\footnotesize\ttfamily]
lemma ContainmentVerificationSoundness(
    c: Constants,
    trace: seq<Variables>,
    events: seq<Event>)
    requires ValidTrace(c, trace, events)
    ensures forall i, j ::
        0 <= i < |trace| && 0 <= j < |trace[i].read_paths|
        ==> StartsWith(trace[i].read_paths[j], c.workspace_root)
    ensures forall i, j ::
        0 <= i < |trace| && 0 <= j < |trace[i].tool_calls|
        ==> trace[i].tool_calls[j] in c.allowed_tools
    ensures forall i :: 0 <= i < |trace| ==> trace[i].step_count <= c.max_steps
\end{lstlisting}

The lemma's body invokes three components. \texttt{TraceRefinesSpec} performs a per-step induction applying (R1) and (R2) to lift a concrete trace to a havoc trace of $\mathcal{T}_M$. \texttt{SpecSafetyAlongTrace} mechanizes abstract havoc safety by establishing \texttt{BoundarySpec.Safety} pointwise via \texttt{InitSatisfiesSafety} and \texttt{SafetyPreserved}. \texttt{SpecSafetyLiftsToImpl}, applied pointwise at each trace position, discharges (R3) by translating abstract \texttt{Safety} on \texttt{VariablesAbstraction(c, trace[i])} to the corresponding concrete safety conjuncts on \texttt{trace[i]}.

\section{Schema-Mismatch Specification Diff}
\label{app:diff}

Section~\ref{sec:empirical} reports a qualitative divergence between barrier and no-barrier modes when the trusted property template's \texttt{Variables} schema does not include a field corresponding to a witness Flow's privileged effect. We exhibit the divergence here by excerpting the synthesized \texttt{NextStep} clause for \texttt{ReadDocNode} on the dispatch path \texttt{05\_ReadDocNode-BoundedLoopNode} of the agentic retrieval-augmented generation Flow. In this Flow, the privileged effect is the append of a document identifier to a \texttt{docs\_read} sequence gated by an \texttt{allowed\_docs} allowlist. Neither field exists in the trusted template, whose corresponding sequence-and-allowlist pair is \texttt{tool\_calls} and \texttt{allowed\_tools}.

\subsection{Barrier mode: renaming homomorphism}
\label{app:diff-barrier}

With no implementation context, the property-derivation step encodes the Flow's privileged effect under the template's vocabulary by treating \texttt{ReadDocNode} as a write to \texttt{tool\_calls} guarded by membership in \texttt{allowed\_tools}. The corresponding \texttt{ToolAllowlisted} corollary of \texttt{Safety} is non-vacuously discharged, but the predicate it discharges is the renamed property ``every appended document is in the allowlist,'' stated in the template's vocabulary.

\begin{lstlisting}[style=dafny,breaklines=true,basicstyle=\footnotesize\ttfamily]
// Witness: tool_calls (models docs_read)
&& (node == ReadDocNode ==>
    |v'.tool_calls| == |v.tool_calls| + 1
    && v.tool_calls == v'.tool_calls[..|v.tool_calls|]
    && v'.tool_calls[|v.tool_calls|] in v.allowed_tools)
&& (node != ReadDocNode ==> v'.tool_calls == v.tool_calls)
\end{lstlisting}

\subsection{No-barrier mode: silent abstention}
\label{app:diff-nobarrier}

With implementation access, the property-derivation step recognizes that \texttt{docs\_read} is a distinct field with no analogue in the template's \texttt{Variables} schema, and synthesizes a \texttt{NextStep} that holds the template's \texttt{tool\_calls} unchanged at every node:

\begin{lstlisting}[style=dafny,breaklines=true,basicstyle=\footnotesize\ttfamily]
// Witness: read_paths unchanged (no file-read node in this flow)
&& v'.read_paths == v.read_paths

// Witness: tool_calls unchanged (no tool-allowlist gate in this flow)
&& v'.tool_calls == v.tool_calls
\end{lstlisting}

\noindent The \texttt{ToolAllowlisted} corollary then quantifies over an always-empty sequence and is vacuously true along every reachable state.

\subsection{Gate outcomes}
\label{app:diff-gates}

Both bundles pass all three gates. Resolution and type check within the thirty-second timeout pass, the permissive-stub vacuity mutation fails verification, and the seeded discrimination mutation also fails verification. Per-path Dafny verification on the agentic retrieval-augmented generation runs took 1.2--2.3 seconds. The barrier mode's \texttt{ToolAllowlisted} predicate is non-vacuously discharged in the template's vocabulary as a corollary of the renaming. The no-barrier mode's \texttt{ToolAllowlisted} predicate is vacuously true because \texttt{tool\_calls} is empty along every reachable state. The gates do not catch the no-barrier case because they audit the protocol-invariant portion of the proof rather than whether each conjunct of the safety predicate is non-vacuously discharged. This is the property-completeness gap discussed in \S\ref{sec:limitations}.

\end{document}